\documentclass[sigconf]{aamas}  

\usepackage{booktabs}

\usepackage{flushend}
\setcopyright{ifaamas}  
\acmDOI{doi}  
\acmISBN{}  
\acmConference[AAMAS'20]{Proc.\@ of the 19th International Conference on Autonomous Agents and Multiagent Systems (AAMAS 2020), B.~An, N.~Yorke-Smith, A.~El~Fallah~Seghrouchni, G.~Sukthankar (eds.)}{May 2020}{Auckland, New Zealand}  
\acmYear{2020}  
\copyrightyear{2020}  
\acmPrice{}  
\usepackage{graphicx}  
\usepackage[utf8]{inputenc} 
\usepackage{url}            
\usepackage{booktabs}       
\usepackage{amsfonts}  
\usepackage{amsmath}
\usepackage{amssymb}
\usepackage{dsfont}
\usepackage[
  separate-uncertainty = true,
  multi-part-units = repeat
]{siunitx}
\usepackage{nicefrac}       
\usepackage{microtype}      
\usepackage{multirow}
\usepackage{graphicx}
\usepackage{xcolor}
\usepackage{diagbox}
\usepackage{algorithmicx}
\usepackage{algpseudocode}
\usepackage{subcaption}
\usepackage[english]{babel}
\usepackage{amsthm}
\usepackage{algorithm}
\usepackage{lipsum}

\begin{document}
\title{ExTra: Transfer-guided Exploration }
\titlenote{Equal contribution by A. Santara and R. Madan. R. Madan worked on this project while he was a student at Indian Institute of Technology Kharagpur. A. Santara is currently working at Google Research.}
\titlenote{Published as an Extended Abstract.}

\author{Anirban Santara}
\affiliation{%
 \institution{Indian Institute of Technology Kharagpur}
 \city{Kharagpur} 
 \state{WB, India}
}
\email{nrbnsntr@gmail.com}

\author{Rishabh Madan}
\affiliation{%
 \institution{University of Washington}
 \city{Seattle} 
 \state{Washington}
}
\email{rishabhmadan96@gmail.com}

\author{Pabitra Mitra}
\affiliation{%
 \institution{Indian Institute of Technology Kharagpur}
 \city{Kharagpur} 
 \state{WB, India}
}
\email{pabitra@cse.iitkgp.ernet.in}

\author{Balaraman Ravindran}
\affiliation{%
 \institution{Indian Institute of Technology Madras}
 \city{Chennai} 
 \state{TN, India}
}
\email{ravi@cse.iitm.ac.in}

\begin{abstract}  
In this work we present a novel approach for transfer-guided exploration in reinforcement learning that is inspired by the human tendency to leverage experiences from similar encounters in the past while navigating a new task. Given an optimal policy in a related task-environment, we show that its bisimulation distance from the current task-environment gives a lower bound on the optimal advantage of state-action pairs in the current task-environment. Transfer-guided Exploration (ExTra) samples actions from a Softmax distribution over these lower bounds. In this way, actions with potentially higher optimum advantage are sampled more frequently. In our experiments on gridworld environments, we demonstrate that given access to an optimal policy in a related task-environment, ExTra can outperform popular domain-specific exploration strategies viz. epsilon greedy, Model-Based Interval Estimation -- Exploration Bonus (MBIE-EB), Pursuit and Boltzmann in rate of convergence. We further show that ExTra is robust to choices of source task and shows a graceful degradation of performance as the dissimilarity of the source task increases. We also demonstrate that ExTra, when used alongside traditional exploration algorithms, improves their rate of convergence. Thus it is capable of complementing the efficacy of traditional exploration algorithms.
\end{abstract}

\maketitle
\section{Introduction}
While attempting to solve a new task human beings tend to take actions motivated by similar situations faced in the past. 
These actions often happen to be good starting points even if the prior experiences are not in the exact same task-environment. A Reinforcement Learning (RL) agent learns by trial and error in an environment using reward signals that are indicative of progress or accomplishment of the target task \citep{sutton2018reinforcement}. The agent uses two policies to act in the environment during the learning phase. The policy that the agent learns to exploit for solving the target task is known as the target policy. The agent also has a behavioral policy that it uses for exploration in order to find better solutions for the target policy. The sample efficiency and convergence time of an RL algorithm heavily depends on the exploration method used by the behavioral policy. A large body of research in reinforcement learning has been dedicated to the formulation of sample-efficient algorithms. Some of the notable developments include count based exploration \cite{strehl2008analysis}, curiosity driven exploration \cite{gregor2014curiosity}, optimism in the face of uncertainty \cite{kearns2002near}, Thompson sampling or posterior sampling and bootstrapped methods \citep{osband2016deep}, parameter space exploration \citep{fortunato2017noisy} and intrinsic motivation \citep{oudeyer2009intrinsic}. However these methods are based on heuristics specific to the current environment and do not use any prior experiences of the agent in other environments.\\ 

The motivation of this work is to formulate an exploration method that uses prior experiences of an agent at similar tasks in other environments for improving the efficiency of exploration in the current task-environment. Prior work includes learning \emph{action priors} from one or more source tasks to identify useful actions or affordances for the target task \cite{abel2015goal}, policy and value function initialization \cite{abel2018policy}, policy reuse \cite{fernandez2006probabilistic}, policy transfer through reward shaping \cite{brys2015policy} and transfer of samples \cite{tirinzoni2019transfer}. Recent work in the Meta Reinforcement Learning community have approached this problem from the perspectives of representation learning  \cite{vezzani2019learning} and introduction of structured stochasticity via a learned latent space \cite{gupta2018meta}. However most of these methods are limited to using source tasks that share the same state and action spaces as the target task. In this paper, we consider the more general setting of cross-domain transfer where the source environment can have different state and action spaces. Cross-domain transfer in RL has been studied by \cite{taylor2007cross} and \cite{joshi2018cross} who propose methods involving transfer of rules and adaptation of the source policy respectively. We take a different approach in our paper that is based on the theory of bisimulation based policy transfer \cite{taylor2009transfer,castro2010using}. Bisimulation metric measures the distance between states in terms of their long term behavior. Bisimulation and its close analogue MDP Homomorphism \cite{ravindran2002model,ravindran2003smdp} have been extensively used for transfer learning in RL \cite{castro2010using,soni2006using,sorg2009transfer,taylor2009bounding} and provide strong theoretical performance guarantees. 
Given a source environment with a known optimal policy for a related task and its lax-bisimulation distance function \citep{taylor2009transfer} from the current task-environment, we derive a lower bound on the optimal advantage of state-action pairs of the current environment. During exploration in the current environment, the agent samples actions from a Softmax distribution over these lower bounds. Thus we use transfer learning to guide the exploration in a way that actions with potentially higher advantage are chosen with higher probability. We call this algorithm \textit{Tra}nsfer-guided \textit{Ex}ploration, abbreviated as ExTra. Our experiments demonstrate that ExTra can achieve significant gains in rate of convergence of Q-learning over exploration methods that only use information associated to the current environment.\\


Our contributions in this paper can be summarized as follows:
\begin{itemize}
    \item We prove that transfer via bisimulation \citep{castro2010using} maximizes a lower bound on the optimal advantage of target actions.
    \item We use this bisimulation distance based lower bound to formulate the Transfer-guided Exploration (ExTra) algorithm as a means of using prior experiences for accelerating reinforcement learning in new environments.
    \item We demonstrate that given the optimal policy from a related task-environment and its lax-bisimulation distance function, ExTra achieves faster convergence compared to exploration methods that only use local information. Additionally, ExTra is robust to source task selection with predictable graceful degradation of performance and can complement traditional exploration methods by improving their rates of convergence.
\end{itemize}

The rest of the paper is organized as follows. We briefly describe the essential theoretical concepts used in this paper in Section \ref{sec:background}. Then we introduce ExTra in Section \ref{sec:proposed_methodology} and present the results of comparison with traditional exploration methods in Section \ref{sec:experimental_results}. We conclude the paper with a summary of the proposed method and scope of future work in Section \ref{sec:conclusion}.
\section{Background}
\label{sec:background}
In this section we present a brief introduction to the essential theoretical concepts used in this paper. A Markov Decision Process (MDP) is a discrete time stochastic control process that is commonly used as a framework for reinforcement learning \cite{sutton2018reinforcement}. An MDP is defined as $\mathcal{M} = \langle S, A, P, R \rangle$ where $S$ is the set of states of the environment, $A$ is the set of actions available to the agent, $P : S \times A \times S \rightarrow [0, 1]$ is the transition function that gives a probability distribution over next states for each state, action pair and $R : S\times A \rightarrow \mathbb{R}$ is the reward function for the task at hand. 
A policy is defined as a function $\pi:S \rightarrow A$ that returns an action for a given state. When an agent executes a policy $\pi$, it generates a trajectory $\tau = \langle s_1, a_1, s_2, a_2, \dots, s_T \rangle$ and the 
accumulated reward obtained in the trajectory is written as $\mathcal{R}(\tau|\pi)$. The goal of RL is to find the optimal policy $\pi^*$ that maximizes the expectation of $\mathcal{R}(\tau|\pi)$.
\begin{equation}
\label{eqn:objective}
    \pi^* = arg\max_{\pi}\mathbb{E}_{\tau\sim\pi}[\mathcal{R}(\tau|\pi)]
\end{equation}

The state action value function $Q^\pi(s, a)$ of a policy $\pi$ is defined as the expected cumulative reward that an agent can get by starting at a state $s$, taking action $a$ and henceforth following the policy $\pi$. 
The optimum state-action value function $Q^*(s,a)$ is defined as $Q^*(s,a) = \max_{\pi}Q^\pi(s,a) \;\; \forall s\in S, a \in A$ and the optimal policy $\pi^*$ can be written in terms of $Q^*(s,a)$ as:
\begin{equation}
    \pi^*(s) = arg\max_{a\in A} Q^*(s,a)
\end{equation}
The optimal value function, $Q^*(s,a)$ can be obtained by value iteration using the Bellman equation \cite{sutton2018reinforcement}:
\begin{equation}
    Q^{t+1}(s,a) = \sum_{s'} P(s, a, s') (R(s, a) + \gamma \max_{b} Q^t(s',b))
\end{equation}
As $t\rightarrow \infty, Q^t(s,a) \rightarrow Q^*(s,a)$. 
We use value iteration to obtain the optimal policies in all our experiments in this paper. In the next two subsections we introduce exploration and transfer methods in RL.

\subsection{Exploration in Reinforcement Learning}
A reinforcement learning agent learns through trial and error in the environment. At any step of decision making, the agent either "exploits" the best policy it has learned or "explores" other actions in search of a better strategy. Balancing exploration and exploitation is a key challenge in RL 
and a large body of literature has been dedicated to the formulation of strategies that address this dilemma. Exploration strategies can be widely classified into two categories: \textit{directed} and \textit{undirected} \cite{thrun1992efficient}. While directed exploration methods utilize exploration specific knowledge collected online, undirected exploration methods are driven almost purely by randomness with the occasional usage of estimates of utility of a state-action pair \cite{thrun1992efficient,tijsma2016comparing}. Popular undirected exploration algorithms include random walk \cite{mozer1990discovering}, $\epsilon$-greedy \cite{whitehead1991learning,sutton2018reinforcement} and softmax or Boltzmann exploration \cite{sutton1990integrated,cesa2017boltzmann}. On the other hand, some notable directed exploration methods are count-based \cite{wiering1998efficient,sato1988learning}, error-based  \cite{schmidhuber1991adaptive,thrun1992active}, and recency based \cite{wiering1998efficient}. In the rest of this section, we briefly describe the exploration algorithms that are used as baselines in the current work. For a detailed review of exploration algorithms in RL, please refer to \cite{tijsma2016comparing}.

\subsubsection{$\mathbf{\epsilon}$\textbf{-greedy}} In $epsilon$-greedy exploration, the agent explores by choosing random actions with probability $\epsilon$ and follows the learned policy greedily the rest of the time.

\subsubsection{\textbf{Model Based Interval Estimation - Exploration Bonus (MBIE-EB)}} MBIE-EB \cite{strehl2005theoretical} is a count-based exploration algorithm that supplies the agent with count-based reward bonuses for favouring exploration of less visited states and actions. The reward bonus is calculated as:
\begin{equation}
    r_{bonus}(s,a) = \frac{\beta}{\sqrt{n(s,a)}}
\end{equation}
where $n(s,a)$ is the number of times the agent chose the state action pair $(s,a)$.

\subsubsection{\textbf{Pursuit}}
Pursuit \cite{sutton2018reinforcement} is an undirected exploration algorithm for Multi-Arm Bandits, adapted for MDP by \cite{tijsma2016comparing}. In Pursuit, the agent follows a stochastic policy $\pi(s,a)$. 
After the update step, $t$, if $a^*_{t+1} = arg\max_a Q_{t+1}(s_t, a)$, Pursuit updates $\pi_{t+1}(s_t, a)$ as follows:

\begin{equation}
    \pi_{t+1}(s_t, a) =
        \begin{cases}
            \pi_t(s_t, a) + \beta [1 - \pi_t(s_t, a)], &\quad \text{if}\; a = a^*_{t+1}\\
            \pi_t(s_t, a) + \beta [0 - \pi_t(s_t, a)], &\quad \text{if}\; a \ne a^*_{t+1}\\
        \end{cases}
\end{equation}
where $\beta$ is a hyperparameter.

\subsubsection{\textbf{Boltzmann exploration}}
Boltzmann or Softmax exploration \cite{sutton2018reinforcement} is another undirected exploration algorithm in which 
the probabilities of the different actions are assigned by a Boltzmann distribution over 
the state-action value function $Q_t(s_t,a)$.
\begin{equation}
    \pi(s_t,a) = \frac{e^{Q_t(s_t,a)/T}}{\sum_{i=1}^m e^{Q_t(s_t, a^i)/T}}
\end{equation}

\subsection{Transfer in Reinforcement Learning}
Transfer learning aims to achieve generalization of skills across tasks by using knowledge gained in one task to accelerate learning of a different task. In reinforcement learning, the transferred knowledge can be high level, such as, rules or advice, sub-task definitions, shaping reward, or low level, such as, task features, experience instances, task models, policies, value functions and distribution priors. Most transfer algorithms for RL make certain assumptions about the relationship between the source and target MDPs. Taylor et al. \cite{taylor2009transfer} gives a comprehensive overview. In this work, we study the general case of transfer of knowledge between MDPs with discrete state and action spaces and make no assumptions about their structures or relationship. We use the bisimulation transfer framework of \cite{castro2010using} as it provides a principled method of transfer for the general setting and a way to estimate the relative goodness of actions under the transferred model \cite{ferns2004metrics,taylor2009bounding}.

\subsection{Transfer using bisimulation metric}
\label{sec:bisim_transfer}
Bisimulation, first introduced for MDP by Givan et al. \cite{givan2003equivalence}, is a relation that draws equivalence between states of an MDP that have the same long-term behavior. Bisimulation is equivalent to the theory of MDP homomorphism \cite{ravindran2002model,ravindran2003smdp} that studies equivalence relations based on reward structure and transition dynamics. For an MDP, $\mathcal{M} = \langle S, A, P, R \rangle$, a relation $E\subset S \times S$ is a bisimulation relation if, for $s_1, s_2 \in S$, whenever $s_1 E s_2$, the following conditions hold:
\begin{eqnarray*}
    \forall a\in A,\quad R(s_1, a) &=& R(s_2, a)\\
    \forall a\in A, \forall C \in S/E,\quad \sum_{s'\in C} P(s_1, a)(s') &=& \sum_{s'\in C} P(s_2, a)(s')
\end{eqnarray*}

\noindent $S/E$ is the set of equivalence classes induced by $E$ in $S$. Ferns et al. \cite{ferns2004metrics} proposed bisimulation metric as a quantitative analogue of the bisimulation relation that can be used as a notion of distance between states of an MDP. Under this metric, two states are bisimilar if their bisimulation distance is zero. The higher the distance, the more dissimilar the states are. In order  to measure the bisimulation distance between state-action pairs of different MDPs, Taylor et al. \cite{taylor2009bounding} introduced the lax bisimulation metric. Considering two MDPs, $\mathcal{M}_1 = \langle S_1, A_1, P_1, R_1 \rangle$ and $\mathcal{M}_2 = \langle S_2, A_2, P_2, R_2 \rangle$ the lax bisimulation metric is defined as follows.\\

\begin{definition}
\textbf{Lax bisimulation metric} \cite{taylor2009bounding,castro2010using}: Let $M$ be the set of all semi-metrics on $S_1 \times A_1 \times S_2 \times A_2$. $\forall s_1\in S_1, a_1\in A_1, s_2\in S_2, a_2\in A_2, d\in M$, $F:M\rightarrow M$ is defined as:
\begin{equation*}
\begin{split}
    F(d)(s_1, a_1, s_2, a_2) = c_R |R_1(s_1, a_1)-R_2(s_2, a_2)| +\\
    \qquad c_T T_K(d')(P_1(s_1, a_1), P_2(s_2, a_2))
\end{split}
\end{equation*}

\noindent Where $d'(s_1, s_2) = \max(\max_{a_1\in A_1} \min_{a_2\in A_2} d((s_1, a_1), (s_2, a_2)), \\
 \min_{a_1\in A_1} \max_{a_2\in A_2} d((s_1, a_1), (s_2, a_2)))$ and $T_K(d')(P_1(s_1, a_1),\\ P_2(s_2, a_2))$ is the Kantorovich distance \cite{gibbs2002choosing} between the transition probability distributions. $c_R, c_T \in \mathbb{R}$ are  tunable hyperparameters representing relative weightages for the reward and Kantorovich components. $F$ has a least-fixed point $d_L$ and $d_L$ is called the Lax Bisimulation Metric between $\mathcal{M}_1$ and $\mathcal{M}_2$.\\
\end{definition}

Algorithm \ref{algo:bisim_transfer2} gives the pseudo-code of the bisimulation based policy transfer algorithm of Castro et al. \cite{castro2010using}. As only optimal actions of the source domain are the ones that get transferred, Castro et al. \cite{castro2010using} define a simplified notation for the lax bisimulation distance as follows:

\begin{equation}
    d_{\approx}(s_1, (s_2, a_2)) = d_L((s_1, \pi_1^*(s_1)), (s_2, a_2))
\end{equation}

\noindent In step $6$, $LB(s_1, s_2)$ refers to the lower bound on the optimal $Q$-value of lax-bisimulation transferred actions defined in Corollary $11$ of Castro et al. \cite{castro2010using}. In step $8$, the matching state in the source environment, $s_{match}$, is assigned such that it maximizes this lower bound. As Castro et al. note in \cite{castro2010using}, this is a superior choice compared to the state that minimizes bisimulation distance alone. The transferred action, $a_2^T$, is the one that minimizes the bisimulation distance from the matching source state, $s_{match}$, and its optimal action, $\pi_1^*(s_{match})$.\\

\noindent Depending upon how $d_\approx(s_1, (s_2, a_2))$ and $d_\approx'(s_1, s_2)$ are related, Castro et al. \cite{castro2010using} define two variants of the lax bisimulation metric -- optimistic and pessimistic.
\begin{equation}
    d_\approx'(s_1, s_2) = \begin{cases}
    \max_{a_2\in A_2} d_\approx(s_1, (s_2, a_2)) \quad \text{if pessimistic}\\
    \min_{a_2\in A_2} d_\approx(s_1, (s_2, a_2)) \quad \text{if optimistic}
    \end{cases}
\end{equation}

\begin{algorithm}[H]
    \centering
    \caption{Bisimulation based policy transfer \cite{castro2010using}}\label{algorithm}
    \begin{algorithmic}[1]
    \For{All $s_1 \in S_1$, $s_2 \in S_2$, $a_2 \in A_2$}
        \State Compute $d_\approx (s_1, (s_2, a_2))$
    \EndFor
    \For{All $s_2 \in S_2$}
        \For{All $s_1 \in S_1$}
            \State $LB(s_1, s_2) \leftarrow V_1^*(s_1) - d_\approx'(s_1, s_2)$
        \EndFor
        \State $s_{match} \leftarrow arg\max_{s_1\in S_1} LB(s_1, s_2)$
        \State $a_2^T \leftarrow arg\min_{a_2\in A_2} d_\approx(s_{match}, (s_2, a_2))$
    \EndFor
    \end{algorithmic}
    \label{algo:bisim_transfer2}
\end{algorithm}

\begin{algorithm}[H]
    \centering
    \caption{$\epsilon$-greedy Q-learning with Transfer-guided Exploration (ExTra)}\label{algo:T-ReX-Q}
    \begin{algorithmic}[1]
    \State Given: $d_\approx (s_1, (s_2, a_2)), \quad \forall s_1 \in S_1, s_2 \in S_2, a_2 \in A_2$ 
    \For{All $s_2 \in S_2$}
        \For{All $s_1 \in S_1$}
            \State $LB(s_1, s_2) \leftarrow V_1^*(s_1) - d_\approx'(s_1, s_2)$
        \EndFor
        \State $s_{match} \leftarrow arg\max_{s_1\in S_1} LB(s_1, s_2)$
    \EndFor
    \State step = $0$
    \While{step $<$ MAXSTEPS}
        \State \textbf{with probability} $\epsilon$
        \State \quad $a_2 \sim \pi_{ExTra}(\cdot | s_2, \mathcal{M}_1, \pi_1^*)$
        \State \textbf{with probability} $1-\epsilon$
        \State \quad $a_2 \leftarrow arg\max_{b} Q_2(s_2, b)$
        \State $r$ = $take\_step$($a_2$)
        \State $update\_Q(Q_2(s_2, a_2), r)$
        \State step = step + $1$
    \EndWhile
    \end{algorithmic}
    \label{algo:ExTra-eps-greedy-Q}
\end{algorithm}


\section{Transfer Guided Exploration}
\label{sec:proposed_methodology}
\begin{figure*}[t]
    \centering
    \includegraphics[width=0.9\textwidth]{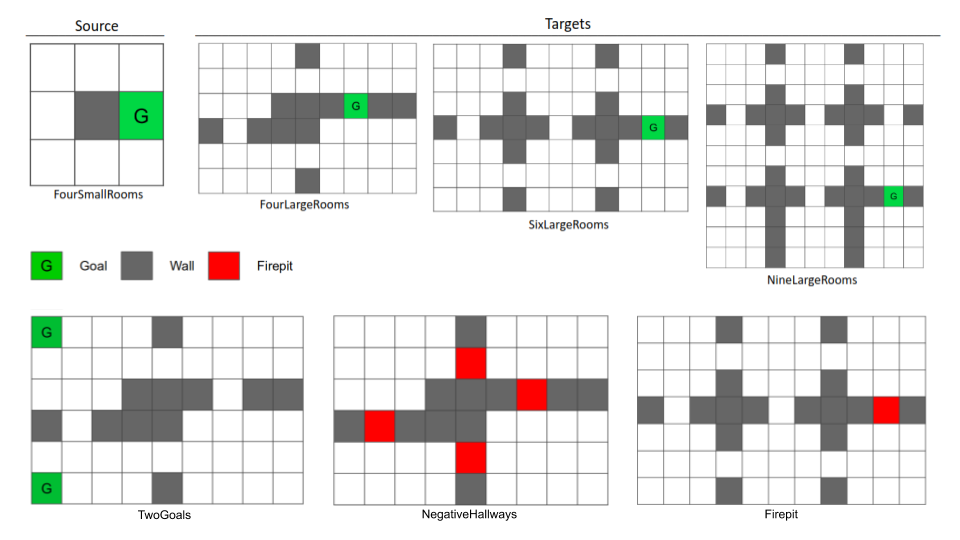}
    \caption{Gridworld environments used in our experiments}
    \label{fig:envs}
\end{figure*}

In this section we present Transfer-guided Exploration (ExTra) as a novel directed exploration method for reinforcement learning that is based on the bisimulation transfer framework of Castro et al. \cite{castro2010using}. The motivation is to use transferred knowledge from the optimal policy in a source domain to accelerate RL in a target domain -- especially during the initial stage, when the domain-specific statistics used by traditional directed exploration methods are yet to be consistently estimated. We first present some results which relate bisimulation distance to the optimal advantage of an action in a target state.\\

\begin{lemma}
\label{lemma:lemma1} $\forall s_1\in S_1$, $\forall s_2\in S_2$, $\forall a_2\in A_2$, $|V_1^*(s_1) - Q_2^*(s_2, a_2)| \le d_\approx(s_1, (s_2, a_2))$.
\end{lemma}
\begin{proof}
\begin{align*}
    |V_1^*(s_1) - Q_2^*(s_2, a_2)| &= |Q_1^*(s_1, \pi_1^*(s_1)) - Q_2^*(s_2, a_2)|\\ \notag
    &\le |R_1(s_1, \pi_1^*(s_1)) - R_2(s_2, a_2)| \\
    & \quad + \gamma T_K(d_\approx)(P_1(s_1, \pi_1^*(s_1)), P_2(s_2, a_2))\\ \notag
    & \text{by similar argument as for Lemma 4 in } 
     \text{\cite{castro2010using}}\\ \notag
    &= d_\approx(s_1, (s_2, a_2))
\end{align*}
\end{proof}

\sloppy
\begin{corollary}
\label{corollary:corollary1} $\forall s_1\in S_1, \forall s_2\in S_2, |V_1^*(s_1) - V_2^*(s_2)| \le d_\approx(s_1, (s_2, \pi_2^*(s_2))).\\
$
\end{corollary}

\noindent The goal of reinforcement learning is to get the agent to take optimal actions and the goal of efficient exploration is to achieve that fast. Hence, efficient exploration algorithms should be able to weigh the actions available in a given state on the basis of their potential optimality. In an MDP, $\mathcal{M} = \langle S, A, P, R \rangle$, the advantage of an action $a\in A$ in a state $s\in S$ with respect to a policy $\pi$ is defined as $A^\pi(s,a) = Q^\pi(s,a)-V^\pi(s)$. It quantifies the relative superiority of action $a$ compared to the expectation under the policy over next actions. The optimal policy maximizes the optimal $Q$ function: $\pi^*(s)=\text{arg}\max_a Q^*(s,a)$. Substituting the expression of $Q$ in terms of advantage, we have: $\pi^*(s) = \text{arg}\max_a A^*(s,a)+V^*(s) = \text{arg}\max_a A^*(s,a)$. Hence the optimal policy also maximizes the optimal advantage function. Given a source MDP, we derive a lower bound on the optimal advantage of available actions in a target state in terms of the lax-bisimulation distance function and use it as a measure of potential optimality.
\begin{theorem} 
\label{theorem:theorem1}
Given MDPs, $\mathcal{M}_1 = \langle S_1, A_1, P_1, R_1 \rangle$ and $\mathcal{M}_2 = \langle S_2, A_2, P_2, R_2 \rangle$ and bisimulation metric $d_\approx:S_1\times S_2\times A_2\rightarrow \mathbb{R}$ we have $\forall s_2\in S_2, \forall a_2\in A_2$
\begin{equation*}
    A_2^*(s_2, a_2) \ge -d_\approx(s_{match}, (s_2, a_2))-\beta(s_2)
\end{equation*}
Where $A_2^*(s_2, a_2)$ is the optimum advantage function in $\mathcal{M}_2$, $s_{match} = arg\max_{s_1\in S_1} V_1^*(s_1) - d_\approx'(s_1, s_2)$ and $\beta(s_2)=d_\approx(s_{match}, (s_2,\pi_2^*(s_2)))$.
\end{theorem}

\begin{proof}
\sloppy
Since, $V_2^*(s_2) = arg \max_{a_2\in A_2} Q_2^*(s_2, a_2)$,  $V_2^*(s_2) - Q_2^*(s_2, a_2) \ge 0$, and we have:
\begin{align*} 
V_2^*(s_2) - Q_2^*(s_2, a_2) &= |V_2^*(s_2) - Q_2^*(s_2, a_2)| \\ \notag
        &=  |(V_2^*(s_2) - V_1^*(s_{match})) +\\
        &\qquad (V_1^*(s_{match}) - Q_2^*(s_2, a_2))|\\ \notag
        &\le |V_2^*(s_2) - V_1^*(s_{match})| +\\
        &\qquad |V_1^*(s_{match}) - Q_2^*(s_2, a_2)|\\ \notag
        &\le d_\approx(s_{match}, (s_2,\pi_2^*(s_2))) + d_\approx(s_{match}, (s_2, a_2))\\
        & \text{(from Lemma \ref{lemma:lemma1} and Corollary \ref{corollary:corollary1})}\\ \notag
        &= \beta(s_2) + d_\approx(s_{match}, (s_2, a_2)) \notag
\end{align*}
$\therefore\quad A_2^*(s_2, a_2) = Q_2^*(s_2, a_2) - V_2^*(s_2) \ge -d_\approx(s_{match}, (s_2, a_2)) - \beta(s_2)$
\end{proof}



Theorem \ref{theorem:theorem1} gives a lower bound on the optimal advantage of an action in a target state in terms of the bisimulation distance to a source MDP. As explained in Section \ref{sec:bisim_transfer}, Castro et al. \cite{castro2010using} assign matching source state $s_{match}$ to target state-action pair $(s_2, a_2)$ following a minimum performance criterion in the form of a lower bound on $Q_2^*(s_2, a_2)$. We choose to use $s_{match}$ in Theorem \ref{theorem:theorem1} to make sure that this minimum performance criterion is satisfied. While other bounds are possible for choices of $s_1 \in S_1, s_1 \neq s_{match}$, they might not be usable, since they do not guarantee the performance bounds identified by Castro et al. \cite{castro2010using}.\\


\begin{corollary}
\label{corollary:corollary2}
The bisimulation transfer algorithm of \cite{castro2010using} maximizes a lower bound on the optimum advantage function $A_2^*(s_2, a_2)$ of the target environment.
\end{corollary}
\begin{proof}

In bisimulation transfer, the transferred action $a_2^T$ for target state $s_2$ is given by:
\begin{align*}
    a_2^T &= arg\min_{a_2\in A_2} d_\approx(s_{match}, (s_2, a_2)) \\
          &= arg\max_{a_2\in A_2} -d_\approx(s_{match}, (s_2, a_2))-\beta(s_2)
\end{align*}
Since, $\beta(s_2)$ is independent of $a_2$.
\end{proof}
Note that Theorem \ref{theorem:theorem1} and Corollaries \ref{corollary:corollary1} and \ref{corollary:corollary2} hold for both optimistic and pessimistic definitions of the lax bisimulation metric \cite{castro2010using}.\\

\begin{definition}
\textbf{Bisimulation Advantage}: Given MDPs, $\mathcal{M}_1 = \langle S_1, A_1, P_1, R_1 \rangle$ and $\mathcal{M}_2 = \langle S_2, A_2, P_2, R_2 \rangle$ and bisimulation metric $d_\approx:S_1\times S_2\times A_2\rightarrow \mathbb{R}$ we define the bisimulation advantage of an action $a_2\in A_2$ in a state $s_2\in S_2$ as: 

\begin{equation}
    A_\approx(s_2, a_2)=-d_\approx(s_{match}, (s_2, a_2))-\beta(s_2)
\end{equation}

\end{definition}

We look to define a probability distribution over target actions that assigns higher probability to actions having higher bisimulation advantages. As, in general, the solution to $arg\min_{a_2} d_\approx(s_{match}, (s_2, a_2))$ is non-unique, under bisimulation, multiple actions in the target environment could potentially map to the optimal action in the source environment. Among all policies that pick at least one of the actions that map to the optimal action under the bisimulation, by the principle of maximum entropy \cite{thomas1991elements} , we choose one that assigns equal probability to all those actions. 
Softmax distribution satisfies this property. Hence we define the behavioral policy in Transfer-guided Exploration (ExTra) as a Softmax distribution over bisimulation advantages.
\begin{equation}
    \pi_{\text{ExTra}}(a_2|s_2, \mathcal{M}_1, \pi_1^*) = \frac{e^{A_\approx(s_2, a_2)}}{\sum_{b\in A_2} e^{A_\approx(s_2, b)}}
\end{equation}


In transfer-guided exploration (ExTra), the agent samples actions from $\pi_{\text{ExTra}}(\cdot | s_2, \mathcal{M}_1, \pi_1^*)$. Since the optimal policy in the target MDP, $\pi_2^*$, is not known during learning, $\beta(s_2)$ can not be known exactly. Since $\beta(s_2) (\ge 0)$ is the same for all actions in a given state $s_2$, replacing $\beta(s_2)$ with a real positive number preserves the order of probability values assigned to different actions by $\pi_{\text{ExTra}}$. If transfer is successful, $\pi_{\text{ExTra}}$ would assign higher probabilities to the optimal actions and thus help the agent to quickly arrive at the optimal policy. However, in the event of an unsuccessful transfer, $\pi_{\text{ExTra}}$ may be biased away from the optimal actions. 
This may cause the agent to remain stuck with the wrong actions for long periods of time. 
In order to help the agent recover from the effect of unsuccessful transfer, we set $\beta = \alpha n$, where $\alpha\in \mathbb{R}^+$ is a tunable hyperparameter and $n$ is the current step number. As $n$ grows, $\pi_{\text{ExTra}}(\cdot | s_2, \mathcal{M}_1, \pi_1^*)$ tends to a uniform distribution over actions, thus annealing the influence of transferred knowledge on exploration with time. This does not hurt the agent's learning in states where the transfer was successful because the agent happens to have explored the optimal actions early on in training in those states. Note that changing $\beta$ does not affect the rate of exploration; instead, it merely changes the shape of the probability distribution from which the agent samples actions during exploration. Algorithm \ref{algo:T-ReX-Q} gives an example use case of ExTra as an alternative to random exploration in $\epsilon$-greedy Q-learning. Please refer to the Appendix section for more ways of using ExTra in conjunction with other traditional exploration methods.

\section{Experimental Results}
\label{sec:experimental_results}

\begin{table*}[t]
\caption{Percentage AuC-MAR of $\epsilon$-greedy Q-learning with ExTra versus traditional exploration methods.}
\label{table:MAR}
\centering
\begin{tabular}{|l|l|l|l|l|l|l|}
\hline
\multicolumn{2}{|l|}{\multirow{2}{*}{Target Environment}} & \multicolumn{5}{c|}{AuC-MAR($\%$)} \\ \cline{3-7} 
\multicolumn{2}{|l|}{} & \multicolumn{1}{c|}{$\epsilon$-greedy} & \multicolumn{1}{c|}{MBIE-EB} & \multicolumn{1}{c|}{Pursuit} & \multicolumn{1}{c|}{Softmax} & \multicolumn{1}{c|}{ExTra} \\ \hline
\multicolumn{2}{|l|}{FourLargeRooms} & $61.08 \pm 3.37$ & $63.91 \pm 1.62$ & $73.96 \pm 2.59$ & $77.64 \pm 1.70$ & $\mathbf{82.79} \pm 1.68$ \\ \hline
\multicolumn{2}{|l|}{SixLargeRooms} & $45.81 \pm 3.07$ & $54.36 \pm 2.69$ & $63.44 \pm 1.44$ & $61.46 \pm 1.31$ & $\mathbf{72.52} \pm 1.28$ \\ \hline
\multicolumn{2}{|l|}{NineLargeRooms} & $43.20 \pm 2.45$ & $45.87 \pm 5.00$ & $61.52 \pm 2.08$ & $53.66 \pm 2.01$ & $\mathbf{64.81} \pm 1.11$ \\ \hline
\end{tabular}
\end{table*}

\begin{figure*}[h!]
    \centering
        \begin{subfigure}{0.3\textwidth}    \includegraphics[width=\linewidth]{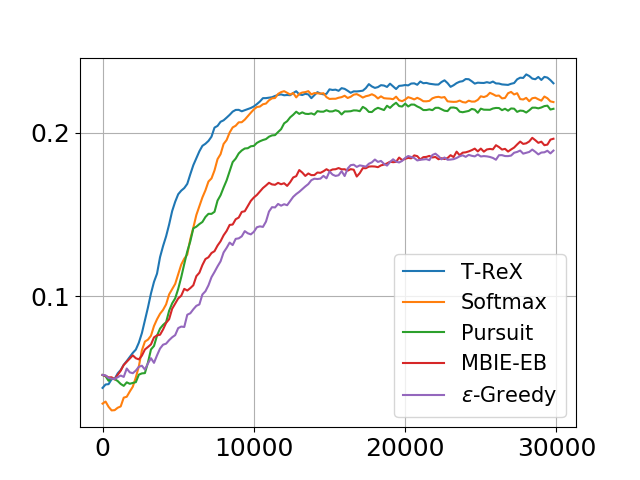}
            \caption{\texttt{FourLargeRooms}}
            \label{fig:1a}
        \end{subfigure}
        \hspace*{\fill}
        \begin{subfigure}{0.3\textwidth}    \includegraphics[width=\linewidth]{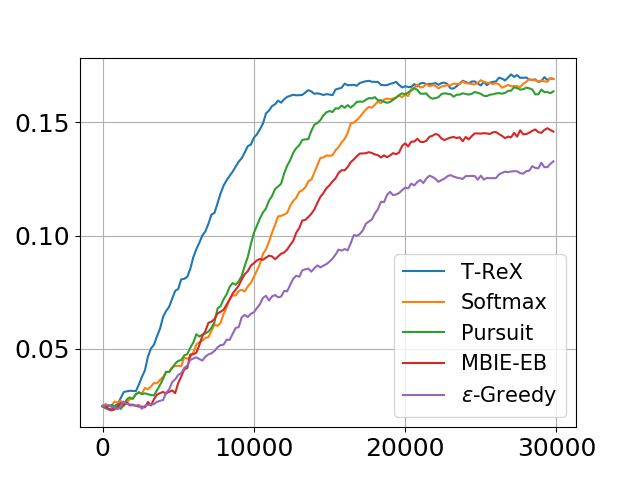}
            \caption{\texttt{SixLargeRooms}}
            \label{fig:1b}
        \end{subfigure}
        \hspace*{\fill}
        \begin{subfigure}{0.3\textwidth}    \includegraphics[width=\linewidth]{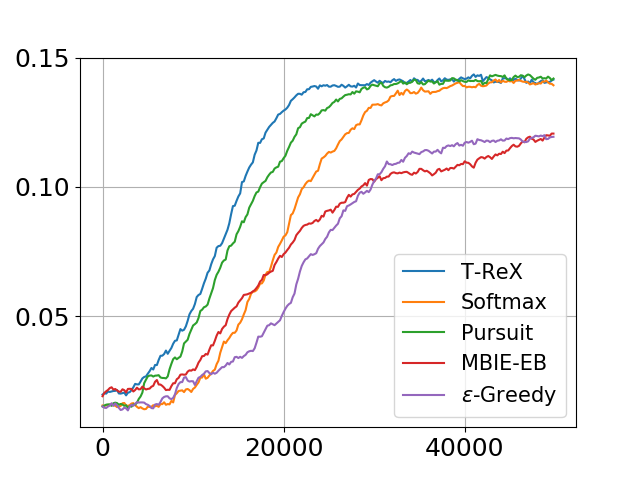}
            \caption{\texttt{NineLargeRooms}}
            \label{fig:1c}
        \end{subfigure}
    \caption{Variation of MAR with steps of training}
    \label{fig:MAR}
\end{figure*}

\begin{figure*}[h!]
    \centering
        \begin{subfigure}{0.22\textwidth}    \includegraphics[width=\linewidth]{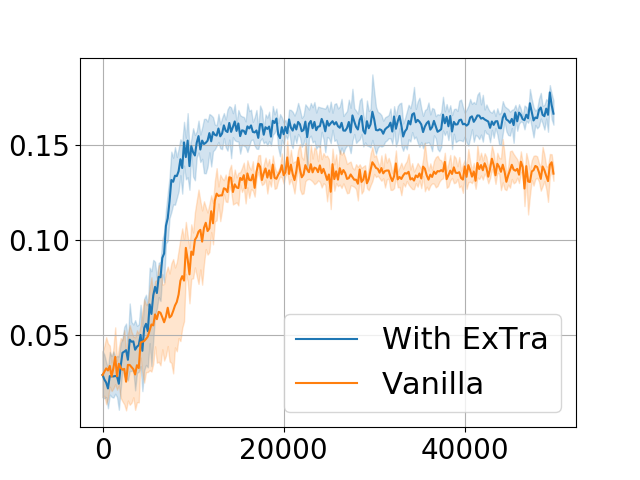}
            \caption{$\epsilon$-greedy}
            \label{fig:4a}
        \end{subfigure}
        \hspace*{\fill}
        \begin{subfigure}{0.22\textwidth}    \includegraphics[width=\linewidth]{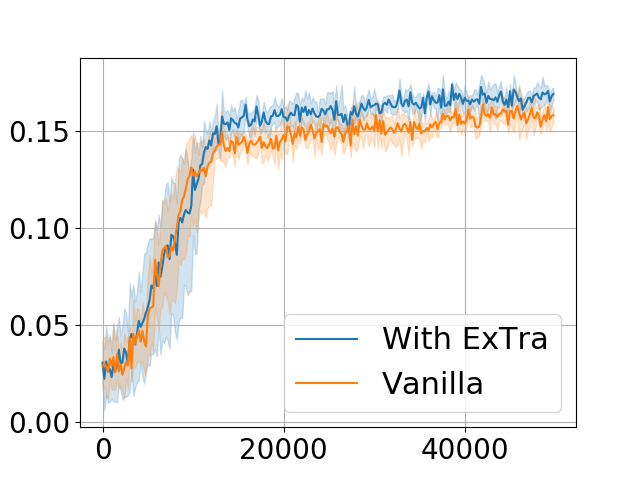}
            \caption{MBIE-EB}
            \label{fig:4b}
        \end{subfigure}
        \hspace*{\fill}
        \begin{subfigure}{0.22\textwidth}    \includegraphics[width=\linewidth]{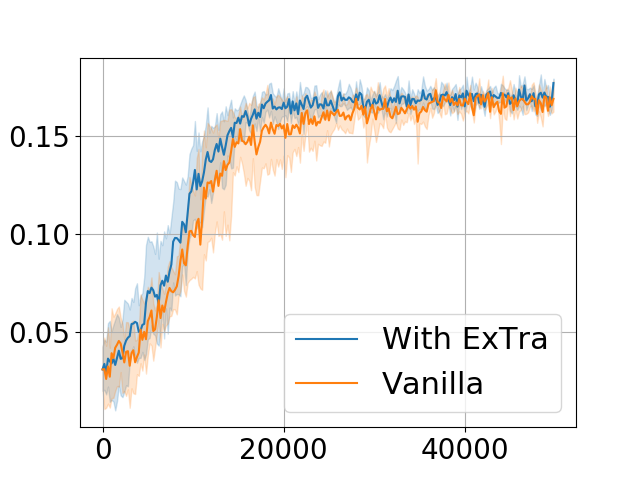}
            \caption{Pursuit}
            \label{fig:4c}
        \end{subfigure}
        \hspace*{\fill}
        \begin{subfigure}{0.22\textwidth}    \includegraphics[width=\linewidth]{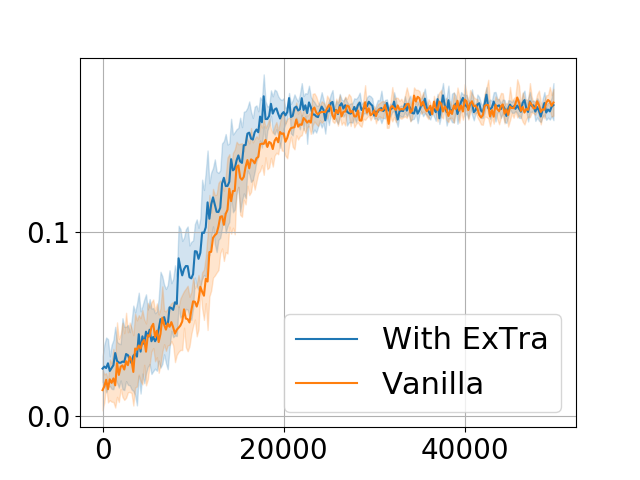}
            \caption{Softmax}
            \label{fig:4d}
        \end{subfigure}
    \caption{Complementary effect of using ExTra in conjunction with traditional exploration methods}
    \label{fig:complimentary}
\end{figure*}
In this section we present an empirical analysis of the performance of ExTra. We aim to address the following questions:
\begin{enumerate}
    \item How does ExTra compare against traditional exploration methods?
    \item How sensitive is ExTra to the choice of source task?
    \item Can ExTra enhance the performance of other exploration algorithms that only use local information?
    \item How does ExTra compare against Bisimulation Transfer?
\end{enumerate}

We use the optimistic definition of the bisimulation metric in all our experiments as Castro et al. \cite{castro2010using} note that it gives superior transfer results than the pessimistic definition.

\subsection{Evaluation metrics}
The goal of this paper is to improve the rate of convergence of RL with the help of external knowledge transferred from a different task. The rate of convergence of an RL algorithm can be judged from a plot of Mean Average Reward (MAR) obtained by the agent over steps of training. Mean Average Reward is defined as follows:

\begin{definition}
\textbf{Mean Average Reward:} The average reward of a trajectory $\tau$ obtained by following a policy $\pi$ is the average of all the rewards received by the agent in the trajectory. Mean Average Reward (MAR) of $\pi$ is the mean of average reward for multiple trajectories rolled out from the policy.
\begin{equation}
    MAR(\pi) = \frac{1}{N} \sum_{i=1}^N \frac{1}{T_i} \sum_{t=1}^{T_i} R(s_t^i, a_t^i) \notag
\end{equation}
Where $T_i$ is the length of the $i^{th}$ trajectory.
\end{definition}

The Area under the MAR Curve (AuC-MAR) is an objective measure of rate of convergence when the highest MAR values achieved asymptotically by all the candidate algorithms are the same \citep{taylor2009transfer}. A higher AuC-MAR implies higher rate of convergence. We use AuC-MAR for comparing the rates of convergence of RL algorithms using ExTra with the baseline methods. For the convenience of comparison, we report AuC-MAR as a percentage of AuC for the optimal policy (whose MAR is represented as a straight horizontal line through the steps of training but not shown in the plots to avoid clutter).\\

For measuring the relative improvement of the rate of convergence achieved by ExTra, we use Transfer Ratio, denoted by $TR$, defined as follows:

\begin{equation}
    TR = \frac{\text{AuC-MAR with ExTra} - \text{AuC-MAR without ExTra}}{\text{AuC-MAR without ExTra}}
\end{equation}




\subsection{Experimental Design, Results and Discussion}
We use stochastic grid-world environments of different levels of complexity (Figures \ref{fig:envs} and \ref{fig:sixtask}) for analysing the viability of ExTra. All of these environments have the common task of avoiding obstacles and reaching a goal. After being initialised with uniform probability from any of the grid-cells, the agent gets a reward of $+1$ on reaching the goal, $-1$ for stepping into a fire-pit and $0$ elsewhere. The agent has four primitive actions: up, down, left and right. When one of the actions is chosen, the agent moves in the desired direction with $0.9$ probability, and with $0.1$ probability it moves uniformly in one of the other three directions or stays in the same place. If the agent bumps into a wall, it does not move and remains in the same state.\\ 

We choose Q-learning with four traditional exploration algorithms viz. $\epsilon$-greedy uniform random exploration, MBIE-EB \cite{strehl2005theoretical}, Pursuit \cite{sutton2018reinforcement} and Softmax \cite{sutton2018reinforcement} as baselines for comparison (see supplementary material for details of these algorithms). For ExTra, we train the optimal policy for the source domain using 
value iteration and calculate the optimistic bisimulation distance $d_\approx(s_1, (s_2, a_2))$, $\forall s_1\in S_1, s_2\in S_2, a_2\in A_2$, tuning $c_R$ and $c_T$ for maximum transfer accuracy. We use the PyEMD library by \cite{Mayner2018pyemd} to calculate earth mover distance for the estimation of $T_K(d)$ \citep{pele2008,pele2009}. Our code is available on GitHub\footnote{https://github.com/madan96/ExTra}. The results reported for each baseline are obtained after rigorous tuning of their respective hyperparameters for the target domain. We report all MAR and AuC-MAR numbers as their mean $\pm$ standard deviation calculated over $10$ different experiments with different random seeds. We tabulate the hyperparameter values in the Appendix section. 


\subsubsection{\textbf{How does ExTra compare against traditional exploration methods?}}
In our first set of experiments, we show that given an optimal policy for a related task-environment, ExTra can obtain faster convergence than traditional exploration methods that only use local information. We choose the \texttt{FourLargeRooms}, \texttt{SixLargeRooms}, and \texttt{NineLargeRooms} environments shown in Figure \ref{fig:envs} as benchmarks. We train Q-learning agents with $\epsilon$-greedy, MBIE-EB, Pursuit and Softmax explorations in these environments as baselines. For ExTra, we choose \texttt{FourSmallRooms} in Figure \ref{fig:envs} as the source environment. We train an $\epsilon$-greedy Q-learning agent ($\epsilon=0.2$) that uses ExTra instead of uniform random for exploration (Algorithm \ref{algo:T-ReX-Q}) on each of the three target environments. Figure \ref{fig:MAR} shows the variation of MAR over steps of training for the baseline as well as our ExTra agents. Table \ref{table:MAR} shows AuC-MAR values. We observe that our ExTra agent consistently achieves faster convergence in all the three environments. This corroborates our claim that ExTra can achieve faster convergence and hence superior sample efficiency if we have access to the optimal policy in a related task-environment. 

\subsubsection{\textbf{How sensitive is ExTra to the choice of source task?}}
\label{sec:sensitivity_to_source_selection}
In this experiment we study the sensitivity of ExTra to the selection of source task. In our \emph{first} study we consider transfer between tasks that share the same state-action space and reward structure but differ in goal positions. We construct $6$ tasks in the \texttt{SixLargeRooms} environment in which the  $i^{th}$ task has its goal at the center of the $i^{th}$ room (see Figure \ref{fig:sixtask}).
\begin{figure}[!h]
    \centering
    \includegraphics[width=0.3\textwidth]{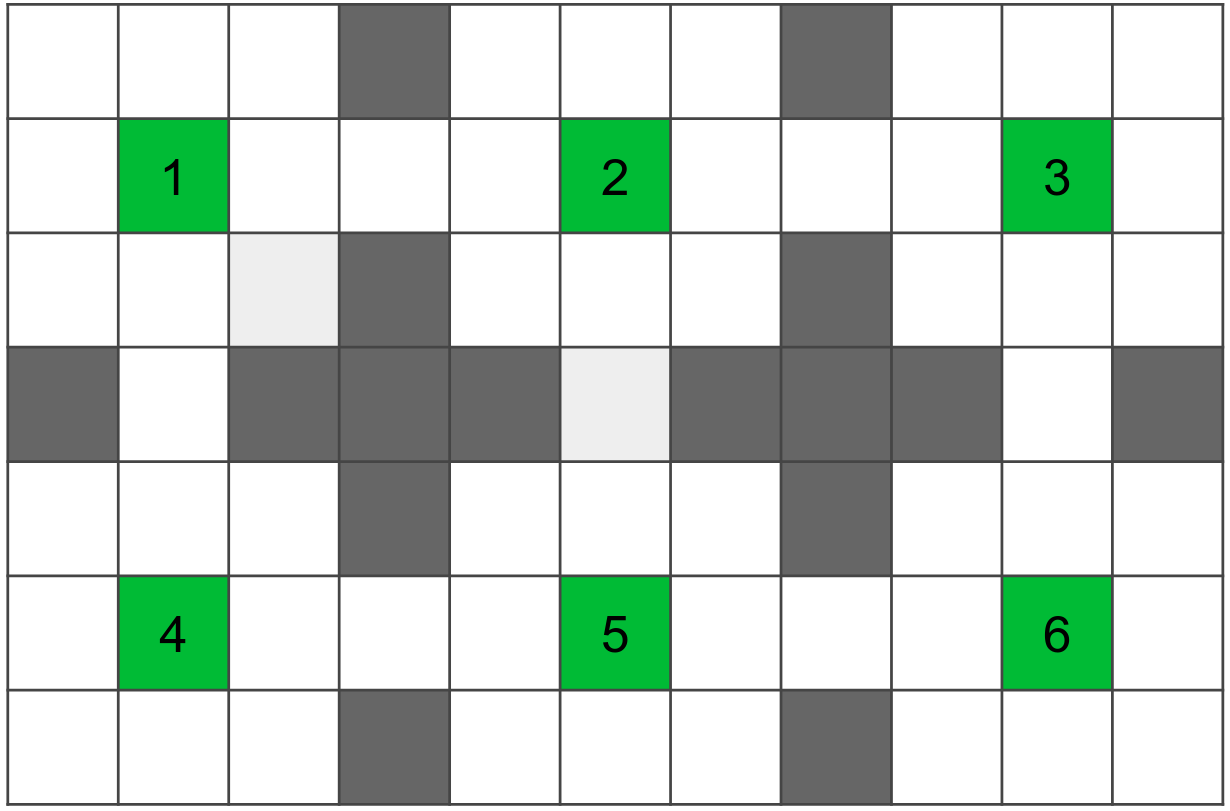}
    \caption{Goal locations in the six tasks used to study the sensitivity of ExTra to the choice of source task.}
    \label{fig:sixtask}
\end{figure} 
We use tasks $1$ through $5$ as source and task $6$ as target. 
In our \emph{second} study, we compare transfer from source tasks that differ in state space (\texttt{FourSmallRooms}), action space (\texttt{3Actions}), reward structure (\texttt{Firepit}, \texttt{NoGoal}, \texttt{NegativeHallways}), goal distribution (\texttt{Two Goals}) and transition dynamics (\texttt{Gravity}). \texttt{3Actions}, \texttt{NoGoal} and \texttt{Gravity} are variants of the \texttt{FourLargeRooms} environment with just three allowed actions (left, right, down), no goal at all and an added probability of $0.1$ to slide downwards, respectively. The rest of the environments are as depicted in Figure \ref{fig:envs}. The target task-environment is \texttt{FourLargeRooms}. We report AuC-MAR values obtained by $\epsilon$-greedy Q-learning with ExTra (Algorithm \ref{algo:ExTra-eps-greedy-Q}) for each of these source tasks in Tables \ref{table:sensitivity_to_source_task} and \ref{table:sensitivity_to_source_task_FOURLARGEROOMS}. \\

\noindent We make the following observations:
\begin{itemize}
    \item In our first study (Table \ref{table:sensitivity_to_source_task}), each of our ExTra agents fetch higher AuC-MAR values than any of the baseline methods thus demonstrating the efficacy of ExTra. Also there is a rough trend of the AuC-MAR values decreasing with increasing distance of goal in the source task. This demonstrates graceful degradation of performance.
    \item In our second study, ExTra beats both $\epsilon-$greedy and MBIE-EB for even the worst case choice of source task - \texttt{Negative Hallways} (compare Table \ref{table:sensitivity_to_source_task_FOURLARGEROOMS} with Row 1 of Table \ref{table:MAR}). This is possible because ExTra is able to leverage knowledge about the transition dynamics of the source MDP (that are identical to the target MDP) even when the reward structures and goal distributions are drastically different. It also reflects the robustness of ExTra to the choice of source task.
    \item We note that in Table \ref{table:sensitivity_to_source_task_FOURLARGEROOMS} the reward structure of the source environment has a more profound effect on the performance of ExTra than transition dynamics or goal distribution. This observation serves as a guide for choosing source environments for bisimulation transfer.
    \item Higher AuC-MAR for \texttt{FourSmallRooms} than \texttt{3Actions} suggests that source MDPs with the same action space as the target are more favourable for ExTra even if the state spaces are different.
\end{itemize}





\begin{table}[]
\caption{Variation of percentage AuC-MAR of $\epsilon$-greedy Q-learning with ExTra exploration for different source task goal positions in \texttt{SixLargeRooms}.}%
\label{table:sensitivity_to_source_task}
\begin{tabular}{c|c|c|}
\cline{3-3}
\multicolumn{2}{c|}{} & AuC-MAR$(\%)$ \\ \hline
\multicolumn{1}{|c|}{\multirow{4}{*}{Baselines}} & $\epsilon$-greedy & $52.92 \pm 3.83$ \\ \cline{2-3} 
\multicolumn{1}{|c|}{} & MBIE-EB & $34.80 \pm 3.98$ \\ \cline{2-3} 
\multicolumn{1}{|c|}{} & Pursuit & $69.05 \pm 2.84$ \\ \cline{2-3} 
\multicolumn{1}{|c|}{} & Softmax & $64.85 \pm 3.29$ \\ \hline
\multicolumn{1}{|c|}{\multirow{5}{*}{Source task \#}} & 1 & $73.70 \pm 2.21$ \\ \cline{2-3} 
\multicolumn{1}{|c|}{} & 2 & $75.73 \pm 2.08$ \\ \cline{2-3} 
\multicolumn{1}{|c|}{} & 3 & $78.54 \pm 2.78$ \\ \cline{2-3} 
\multicolumn{1}{|c|}{} & 4 & $74.59 \pm 2.81$ \\ \cline{2-3} 
\multicolumn{1}{|c|}{} & 5 & $81.94 \pm 1.70$ \\ \hline
\end{tabular}
\end{table}


\begin{table}[]
\caption{Variation of percentage AuC-MAR of $\epsilon$-greedy Q-learning with ExTra exploration for different choices of source task and \texttt{FourLargeRooms} as target.}%
\label{table:sensitivity_to_source_task_FOURLARGEROOMS}
\begin{tabular}{|l|l|c|}
\hline
\multicolumn{1}{|c|}{Source Task} & Difference & AuC-MAR$(\%)$ \\ \hline \hline
TwoGoals & Goal distribution & $77.83 \pm 1.54$ \\ \hline
Firepit & Reward structure & $75.54 \pm 1.59$ \\ \hline
NoGoal & Reward structure & $75.59 \pm 1.81$ \\ \hline
NegativeHallways & \begin{tabular}[c]{@{}l@{}}Reward structure\\ and Goal distribution\end{tabular} & $73.36 \pm 1.62$ \\ \hline
3Actions & Action space &
$77.60 \pm 1.78$ \\ \hline
Gravity & Transition dynamics &
$77.78 \pm 1.56$ \\ \hline
FourSmallRooms & State space &
$82.79 \pm 1.68$ \\ \hline
\end{tabular}
\end{table}




\subsubsection{\textbf{Can ExTra enhance the performance of other exploration algorithms that only use local information?}}
We formulate $\epsilon$-greedy versions of each of our baseline algorithms ($\epsilon=0.5$), where the agent samples actions from $\pi_{ExTra}$ with probability $\epsilon$ $(\epsilon=0.5)$ and follows the main algorithm rest of the time. Please refer to the supplementary material for pseudo-codes of these algorithms. We choose \texttt{SixLargeRooms} as our benchmark environment and \texttt{FourSmallRooms} (Figure \ref{fig:envs}) as source environment for ExTra. Table \ref{table:complimentary} and Figure \ref{fig:complimentary} present the results of these experiments. 
We see that ExTra can provide improved rate of convergence of the traditional methods and hence can indeed be used as a complementary exploration method. 

\begin{table}[!h]
\caption{Comparison of percentage AuC-MAR scores of traditional exploration methods with and without ExTra}
\label{table:complimentary}
\centering
\begin{tabular}{l|c|c|c|}
\cline{2-4}
\multirow{2}{*}{} & \multicolumn{2}{c|}{AuC-MAR$(\%)$ } & \multirow{2}{*}{$TR(\%)$} \\ \cline{2-3}
 & \multicolumn{1}{c|}{vanilla} & \multicolumn{1}{c|}{with ExTra} &  \\ \hline
\multicolumn{1}{|l|}{$\epsilon$-greedy} & $61.91 \pm 1.30$ & $76.10 \pm 1.44$ & $23.41$ \\ \hline
\multicolumn{1}{|l|}{MBIE-EB} & $70.55 \pm 1.30$ & $74.87 \pm 1.68$ & $6.11$ \\ \hline
\multicolumn{1}{|l|}{Pursuit} & $72.80 \pm 2.02$ & $76.87 \pm 1.16$ & $5.59$ \\ \hline
\multicolumn{1}{|l|}{Softmax} & $69.51 \pm 1.60$ & $72.59 \pm 1.14$ & $4.43$ \\ \hline
\end{tabular}
\end{table}

\subsubsection{\textbf{How does ExTra compare against Bisimulation Transfer?}}
To address this question we choose \texttt{FourSmallRooms} as source environment. As target, we choose \texttt{FourLargeRooms} and a modified version of the \texttt{Taxi-v2} environment of OpenAI Gym \cite{brockman2016openai} with $54$ states. The modified \texttt{Taxi-v2} environment has $3$ \textit{rows}, $3$ \textit{columns} and $2$ \textit{drop locations} at the top-left and bottom-right corners. The action space of this environment remains same as the original version, where the agent can choose actions from \texttt{South}, \texttt{North}, \texttt{West}, \texttt{East}, \texttt{Pickup} and \texttt{Drop}. We compare the rates of convergence of $\epsilon-$greedy Q-learning with ExTra and the bisimulation transfer algorithm of Castro et al. \cite{castro2010using} that initializes the Q-matrix with the Q-value of the transferred policy. Figure \ref{fig:bisim_compare} presents the results of the experiments. We see that when the source and target tasks are similar (\texttt{FourSmallRooms} and \texttt{FourLargeRooms}) and bisimulation policy transfer is successful, Q-learning gets an initial jumpstart while ExTra catches up. When the source and target tasks are drastically different (\texttt{FourSmallRooms} and \texttt{Taxi-v2}), the bisimulation distances are large and $\pi_{ExTra}$ tends to a uniform distribution over target actions. As a result, ExTra falls back to vanilla $\epsilon-$greedy Q-learning with uniform sampling. On the other hand, Q-learning initialized with bisimulation transfer has to first recover from the effect of negative transfer using $\epsilon-$greedy uniform exploration before it can start learning. In this case, it does not get any jumpstart and also converges slower than ExTra which does not need to correct for the damage done by negative transfer and proceeds as vanilla $\epsilon-$greedy Q-learning. We observe that, while bisimulation transfer can be both effective and fatal depending on how the source and target tasks are related, ExTra does not negatively affect the learning process even when the source and target task-environments are drastically different.

\begin{figure}[h!]
    \centering
        \begin{subfigure}{0.68\columnwidth}    \includegraphics[width=\linewidth]{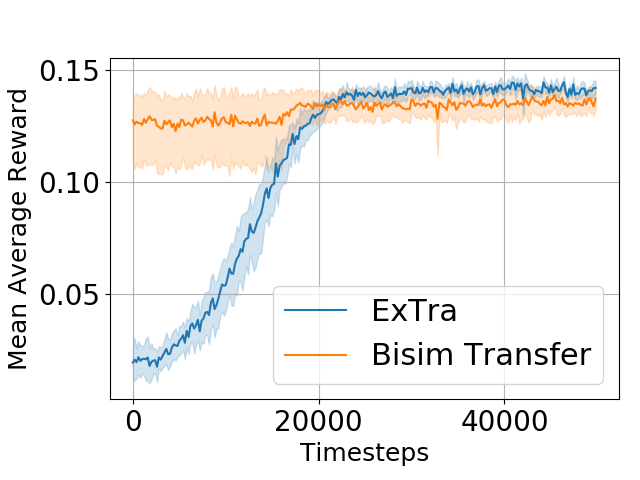}
            \caption{\texttt{FourLargeRooms}}
            \label{fig:1a}
        \end{subfigure}
        \begin{subfigure}{0.70\columnwidth}    \includegraphics[width=\linewidth]{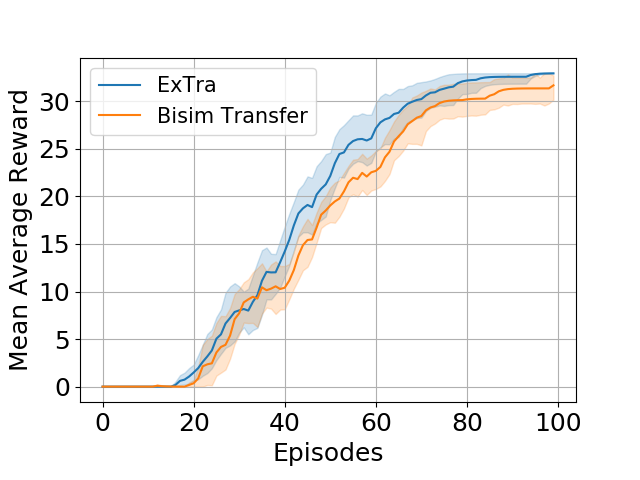}
            \caption{\texttt{Taxi-v2}}
            \label{fig:1b}
        \end{subfigure}
    \caption{Comparison between $\epsilon-$greedy Q-learning with ExTra and the bisimulation transfer algorithm of \cite{castro2010using}.}
    \label{fig:bisim_compare}
\end{figure}

\section{Conclusion}
\label{sec:conclusion}
In this work we investigate the fundamental possibility of using transfer to guide exploration in RL and formulate a novel transfer guided exploration algorithm, ExTra, based on the theory of bisimulation based policy transfer in MDPs. We demonstrate that our method achieves faster convergence compared to traditional exploration methods that only use local information, is robust to source task selection and can complement traditional exploration methods by improving their rate of convergence. We also provide theoretical guarantees in the form of a lower bound on the optimal advantage of an action in the target domain in terms of bisimulation distance from the source environment. In our future work we plan to extend ExTra to larger state-action spaces and continuous control tasks, using differential and sampling-based methods for approximating the lax-bisimulation distance similar to \cite{gelada2019deepmdp}. An interesting direction of research could be to incorporate ExTra within the meta-RL framework for improving the meta-training and meta-testing sample efficiencies by leveraging knowledge from multiple source MDPs.
\section*{Acknowledgements}
We gratefully acknowledge Pablo Samuel Castro of Google for helping us with the optimized implementation of bisimulation distance metrics and clarifying some crucial concepts regarding bisimulation based policy transfer in MDPs. Anirban Santara’s work in this project was supported by Google India under the Google India PhD Fellowship Award.

\bibliographystyle{ACM-Reference-Format}
\bibliography{my_references}  

\clearpage
\appendix
\section*{\Large{Appendix}}
The full code-base is available open-source at \url{https://github.com/madan96/ExTra}

\section{ExTra variants of traditional exploration strategies}
\label{append:trex-variants}
\begin{algorithm}[H]
    \centering
    \caption{ExTra + $\epsilon$-greedy Q-learning with random sampling}\label{ExTra-eps_greedy}
    \begin{algorithmic}[1]
    \State step = $0$
    \While{step < MAXSTEPS}
        \State \textbf{with probability} $\epsilon$
            \State \qquad \textbf{with probability} $\epsilon_{bisim}$
                \State \qquad \quad  $b \sim \pi_{ExTra}(\cdot | t, \mathcal{M}_1, \pi_1^*)$
            \State \qquad \textbf{with probability} $1-\epsilon_{bisim}$
                \State \qquad \quad $b \sim uniform(A_2)$
        \State \textbf{with probability} $1-\epsilon$
            \State \qquad $b \leftarrow arg\max_{b'} Q_2(t, b')$
        \State $r$ = $take\_step$($b$)
        \State $update\_Q(Q_2(t, b), r)$
        \State step = step + $1$
    \EndWhile
    \end{algorithmic}
    \label{algo:ExTra-epsgreedy}
\end{algorithm}

\begin{algorithm}[H]
    \centering
    \caption{ExTra + Softmax}\label{ExTra-Softmax}
    \begin{algorithmic}[1]
    \State step = $0$
    \While{step < MAXSTEPS}
        \State \textbf{with probability} $\epsilon$
            \State \qquad $b \sim \pi_{ExTra}(\cdot | t, \mathcal{M}_1, \pi_1^*)$
        \State \textbf{with probability} $1-\epsilon$
            \State \qquad $b \sim softmax(Q)$  
        \State $r$ = $take\_step$($b$)
        \State $update\_Q(Q_2(t, b), r)$
        \State step = step + $1$
    \EndWhile
    \end{algorithmic}
    \label{algo:ExTra-Softmax}
\end{algorithm}
\begin{algorithm}[H]
    \centering
    \caption{ExTra + Pursuit}\label{ExTra-Pursuit}
    \begin{algorithmic}[1]
    \State step = $0$
    \State $\pi_{pursuit} = Uniform(A_2)$ 
    \While{step < MAXSTEPS}
        \State \textbf{with probability} $\epsilon$
            \State \qquad $b \sim \pi_{ExTra}(\cdot | t, \mathcal{M}_1, \pi_1^*)$
        \State \textbf{with probability} $1-\epsilon$
            \State \qquad $a \leftarrow arg\max_{b'} Q_2(t, b')$
            \State \qquad $update\_\pi_{pursuit}(a)$
            \State \qquad $b \leftarrow Sample(\pi_{pursuit})$
        \State $r$ = $take\_step$($b$)
        \State $update\_Q(Q_2(t, b), r)$
        \State step = step + $1$
    \EndWhile
    \end{algorithmic}
    \label{algo:ExTra-Pursuit}
\end{algorithm}
\begin{algorithm}[H]
    \centering
    \caption{ExTra + MBIE-EB}\label{ExTra-MBIE_EB}
    \begin{algorithmic}[1]
    \State step = $0$
    \State $\pi_{pursuit} = Uniform(A_2)$ 
    \While{step < MAXSTEPS}
        \State \textbf{with probability} $\epsilon$
            \State \qquad $b \sim \pi_{ExTra}(\cdot | t, \mathcal{M}_1, \pi_1^*)$
        \State \textbf{with probability} $1-\epsilon$
            \State \qquad $b \leftarrow arg\max_{b'} Q_2(t, b')$
        \State $r$ = $take\_step$($b$) + $\frac{\beta}{\sqrt{n(s, a}}$
        \State $update\_Q(Q_2(t, b), r)$
        \State step = step + $1$
    \EndWhile
    \end{algorithmic}
    \label{algo:ExTra-MBIEEB}
\end{algorithm}

\newpage
\section{Hyperparameters for baseline exploration strategies}
\label{append:baseline-hyper}
\begin{table}[!h]
\begin{tabular}{|l|l|l|}
\hline
\multicolumn{1}{|c|}{\multirow{2}{*}{$\epsilon$-greedy}} & Q Learning Rate & 0.2 \\ \cline{2-3} 
\multicolumn{1}{|c|}{} & $\epsilon$ & 0.5 \\ \hline
\multirow{2}{*}{Softmax} & Q Learning Rate & 0.2 \\ \cline{2-3} 
 & $\tau$ & 8.1 \\ \hline
\multirow{3}{*}{MBIE-EB} & Q Learning Rate & 0.2 \\ \cline{2-3} 
 & cb-$\beta$ & 0.005 \\ \cline{2-3} 
 & $\epsilon$ & 0.2 \\ \hline
\multirow{2}{*}{Pursuit} & Q Learning Rate & 0.2 \\ \cline{2-3} 
 & $\beta$ & 0.007 \\ \hline
\multirow{3}{*}{ExTra} & Q Learning Rate & 0.5 \\ \cline{2-3} 
 & $\epsilon$ & 0.2 \\ \cline{2-3} 
 & $\alpha$ & 1e-6 \\ \hline
\end{tabular}
\end{table}

\section{Hyperparameters for Optimistic Bisimulation Transfer}
\label{append:bisim-hyper}
\begin{table}[!h]
\begin{tabular}{|p{0.5in}|p{0.5in}|p{0.5in}|p{0.5in}|}
\hline
Transfer Hyperparameters & Four Large Rooms & Six Large Rooms & Nine Large Rooms \\ \hline
$c_R$ & 0.1 & 0.2 & 0.1 \\ \hline
$c_T$ & 0.9 & 0.9 & 0.9 \\ \hline
Threshold & 0.01 & 0.01 & 0.01 \\ \hline
Least fixed point iterations & 5 & 5 & 5 \\ \hline
\end{tabular}
\end{table}
\end{document}